\documentclass[submission,copyright,
creativecommons]{eptcs}
\providecommand{\event}{TARK 2025} 

\usepackage{iftex}

\ifpdf
\usepackage{underscore}         
\usepackage[T1]{fontenc}        
\else
\usepackage{breakurl}           
\fi

\title{On the Complexity of the Grounded Semantics for Infinite Argumentation Frameworks}
\author{Uri Andrews\thanks{This research was supported by the NSF grant DMS-2348792.}
	\institute{Department of Mathematics\\ University of Wisconsin-Madison}
	\email{andrews@math.wisc.edu}
	\and
	Luca San Mauro\thanks{San Mauro is a member of INDAM-GNSAGA.}
	\institute{Dipartimento di Ricerca e Innovazione Umanistica\\ University of Bari, Italy}
	\email{luca.sanmauro@gmail.com}
}
\def\titlerunning{Complexity of the Grounded Semantics for Infinite Argumentation Frameworks}
\def\authorrunning{U.\ Andrews \& L.\ San Mauro}

\usepackage{cite}
\usepackage{amsmath,amssymb,amsfonts, amsthm}

\usepackage{algorithmic}
\usepackage{graphicx}
\usepackage{textcomp}
\usepackage{xcolor, xspace}
\def\BibTeX{{\rm B\kern-.05em{\sc i\kern-.025em b}\kern-.08em
		T\kern-.1667em\lower.7ex\hbox{E}\kern-.125emX}}

\usepackage[symbol]{footmisc}

\usepackage{thmtools}
\usepackage{thm-restate}

\usepackage{tikz}
\usetikzlibrary{arrows.meta, positioning}
\usetikzlibrary{shapes}
\usetikzlibrary{fit}

\newtheorem{theorem}{Theorem}

\newtheorem{definition}[theorem]{Definition}
\newtheorem{lemma}[theorem]{Lemma}

\newtheorem{observation}[theorem]{Observation}
\newtheorem{defn}[theorem]{Definition}
\newtheorem{example}[theorem]{Example}

\newtheorem{notation}[theorem]{Notation}
\newtheorem{corollary}[theorem]{Corollary}

\providecommand{\CK}{{\omega_1^{\text{CK}}}}
\providecommand{\CKD}{{\CK(d)}}
\providecommand{\A}{\mathcal{A}}
\providecommand{\F}{\mathcal{F}}
\providecommand{\G}{\mathcal{G}}

\providecommand{\C}{\mathcal{C}}
\providecommand{\T}{\mathcal{T}}

\providecommand{\demp}{\mathbf{0}}

\providecommand{\uswp}{$\Pi^1_1$ uniformly self-witnessing by paths\xspace}

\providecommand{\badadmissible}{self-defending\xspace}

\renewcommand{\hat}{\widehat}
\renewcommand{\phi}{\varphi}

\providecommand{\att}{\rightarrowtail}
\providecommand{\natt}{\not\rightarrowtail}

\providecommand{\nat}{\mathbb{N}}

\providecommand{\mran}{\text{mran}}
\providecommand{\rank}{\text{rank}}

\begin{document}
	\maketitle

	
	\begin{abstract}
		Argumentation frameworks, consisting of arguments and an attack relation representing conflicts, are fundamental for formally studying reasoning under conflicting information. We use methods from mathematical logic, specifically computability and set theory, to analyze the grounded extension, a widely-used model of maximally skeptical reasoning, defined as the least fixed-point of a natural defense operator. Without additional constraints, finding this fixed-point requires transfinite iterations. We identify the exact ordinal number corresponding to the length of this iterative process and determine the complexity of deciding grounded acceptance, showing it to be maximally complex. This shows a marked distinction from the finite case where the grounded extension is polynomial-time computable, thus simpler than other reasoning problems explored in formal argumentation.
	\end{abstract}
	

	\section{Introduction}
	Over the past three decades, formal argumentation has established itself as a prominent research area within Artificial Intelligence, owing to its versatility in addressing various reasoning tasks. These include nonmonotonic reasoning, multi-agent systems, rule-based systems, and the analysis of debates or dialogues. Formal argumentation provides a unifying framework for representing diverse reasoning approaches, ranging from highly skeptical to more permissive forms of inference (for a comprehensive introduction to this area, see the handbook \cite{Baroni2018-BARHOF}).
	
	At the heart of formal argumentation lies Dung's \emph{abstract argumentation frameworks} (AFs) \cite{dung1995acceptability}, which are modeled as directed graphs, where nodes correspond to arguments, and directed edges represent the attack relations between them.  AFs serve as a common foundational core across various reasoning systems in formal argumentation, with many extensions and refinements, e.g.\ \cite{cayrol2013bipolarity,baumeister2018verification,fazzinga2015complexity}, building upon their structure. Moreover, the simplicity and abstraction of AFs make them a natural testing ground for developing and evaluating ideas before extending them to more sophisticated reasoning paradigms.
	

	
	Although most of the research has focused on finite AFs, the study of infinite AFs is a natural and increasingly popular topic (see, e.g., \cite{baumann2015infinite,baroni2013automata,caminada2014grounded-inf,andrews2024NMR,andrews2024FCR}). A practical reason for studying infinite AFs is that rule-based argumentation formalisms (such as \texttt{ASPIC} \cite{caminada2007evaluation} and \texttt{ASPIC}$^+$ \cite{modgil2013general}) are evaluated using Dung-style AFs and can generate infinitely many arguments even starting from a finite set of rules \cite{strass2013approximating,caminada2007evaluation}. As a theoretical foundation, infinite AFs enable the modeling of phenomena that finite frameworks cannot capture: for example, Dung \cite[\S 3.1]{dung1995acceptability} showed that infinite AFs are needed to model solutions to $n$-person games. For additional examples of the applications and relevance of infinite AFs, we refer the reader to \cite{baroni2013automata}. In the infinite setting, where conclusions may not always be attained in finite time, focus of computational research into reasoning problems shifts to finding appropriate limiting procedures.
	
	
	In this paper, we use methods from mathematical logic, specifically computability theory and set theory, to examine the \emph{grounded extension} in the context of infinite AFs. Introduced by 
	Dung \cite{dung1995acceptability},  the grounded extension identifies the unavoidable conclusions within an argumentation framework. The grounded extension, especially in the infinite context, plays a role in several applications of argumentation theory. For example, in cooperative games, the set of outcomes of a game forms an infinite AF where the grounded extension is exactly the supercore of the cooperative game (see \cite{coopgames,roth}). Our results in this paper apply directly to the problem of determining the supercore of cooperative games.
	
	The grounded extension is formally defined as the smallest fixed-point of the defense function (see Definition \ref{def:characteristic function}), but can also be understood as the outcome of an \emph{iterative generative process}: This process begins by accepting arguments that are unanimously approved (i.e., unattacked). Once these are accepted, any argument attacked by an accepted one is invalidated and removed. Repeating this procedure yields a collection of arguments which should be accepted in (almost) any reasoning setting. This provides a natural limiting procedure to compute the grounded extension.

	For finite AFs, the behavior of this iterative process is straightforward: If the AF considered contains $n$ arguments, then the process stabilizes (and thus terminates) after at most $n$ steps. It follows that  determining membership in the grounded extension is computationally tractable, i.e., solvable in polynomial time. However, for infinite AFs, the situation is markedly different. The iterative generative process may run for infinitely many steps before stabilizing. Dung \cite{dung1995acceptability} noted this, as well as noting that for \emph{finitary} AFs---those where each argument is attacked by only finitely many others--- the process is guaranteed to terminate in $\omega$ many steps (that is, after doing one step for each natural number). Beyond this special case, however, the general behavior of the iterative process remains more complex, and Dung left open the question of how long it might extend. 
	
	In fact, it was not immediately clear that grounded extensions exist 
	without assuming the power of full transfinite induction, which is known to be equivalent to the Axiom of Choice. In work exploring the theoretical aspects of the reasoning notions of formal argumentation,  
	Spanring \cite{spanringThesis} demonstrated that the existence and uniqueness of grounded extensions in infinite AFs follow from Zermelo–Fraenkel set theory (ZF), without invoking the Axiom of Choice.\footnote[2]{In contrast, for other semantics widely studied in formal argumentation---such as the  preferred and the naive ones---Spanring \cite{spanring2014axiom} proved that the
		existence of extensions in these semantics is equivalent to the axiom of choice
		over ZF.} Yet, Spanring's proof is  nonconstructive,  relying on Hartog’s Lemma, which ensures that no set is larger than all ordinals. As a result, the exact length of the iterative process remained undetermined.
	
	We provide a new approach to understanding the grounded extension, rooted in computability theory. This approach allows to determine the exact length of the iterative generative process and to determine the exact computational complexity of the grounded extension. Furthermore, combining this analysis with a set-theoretic forcing argument, we show that not only the Axiom of Choice is unnecessary, but the very weak foundation of Kripke–Platek set theory suffices to define the grounded extension.

	We formalize the question of the length of the iterative generative process by defining, for any given argumentation framework $\F$, its \emph{grounding ordinal} (see Definition \ref{def:grounding ordinal}). We show that the grounding ordinal of a countable AF can be an arbitrarily large countable ordinal. In contrast, for computable AFs---which represent a very natural setting: e.g., many infinite AFs generated in applications are computable---we give a much lower, more concrete, and sharp upper bound.
	
	Beyond analyzing the length of the iterative process, we also determine the algorithmic complexity of the grounded extension itself.
	
	
	\begin{theorem}
		For computable AFs, the decision problem  of determining
		membership in the grounded extension is $\Pi^1_1$-complete.\footnote[4]{This result establishes the complexity of the skeptical acceptance of arguments with respect to the complete semantics for infinite AFs, providing a proof of a theorem announced but left unproven in \cite{andrews2024FCR}.}
	\end{theorem}

	
	$\Pi^1_1$-completeness (see Definition \ref{def:Pi11 completeness}) is the natural analog in the infinite setting of co-NP-completeness. $\Pi^1_1$-completeness of the grounded extension shows that in the infinite setting, the grounded extension is as complex as the other computational problems explored for infinite AFs \cite{andrews2024NMR}. This result highlights a stark contrast with the finite case, where the grounded extension is solvable in polynomial time, and thus simpler than most other reasoning problems explored in formal argumentation.
	Moreover, we establish that a single computable AF suffices to achieve this hardness:
	
	\begin{theorem}
		There is a computable AF $\F$ so that the grounded extension of $\F$ is a $\Pi^1_1$-complete set.
	\end{theorem}
	
	This implies that the Turing degree of a $\Pi^1_1$-complete set can be the Turing degree of a grounded extension in a computable argumentation framework.  This naturally leads to the question: What is the precise characterization of the collection of Turing degrees of grounded extensions in computable AFs? It is natural to conjecture that every Turing degree of a $\Pi^1_1$ set could arise as the degree of the grounded extension of some computable AF; we prove that this is not the case.
	

	The rest of the paper is organized as follows. In Section \ref{sec:background}, we review the necessary background on both argumentation theory and computability theory. In Section \ref{sec:maximially long and hard}, we give examples of computable AFs with grounding ordinal equal to any $\alpha\leq \CK$ (see Definition \ref{definition: omega1CK}), as well as a computable AF where the grounded extension is $\Pi^1_1$-complete. We then shift to the problem of showing that $\CK$ is an upper bound for the grounding ordinal of a computable AF.
	In Section \ref{sec:tree rank analysis of groundedness}, we introduce our main technical tool towards this goal, which is a tree-rank analysis of groundedness or non-groundedness in an argumentation framework. In Section \ref{sec:main results}, we use this analysis to conclude that $\CK$ is an upper bound for the grounding ordinal of any computable AF. In Section \ref{sec:conclusions for non-computable AFs}, we see that we can draw conclusions about noncomputable and even uncountable AFs. Here we establish that Kripke-Platek set theory suffices to define the grounded extension.
	Finally, in Section \ref{sec:Turing degrees missing}, we show that there are $\Pi^1_1$ Turing degrees which are not the degree of a grounded extension in a computable argumentation framework. 
	Due to lack of space, some technical proofs have been relegated to an appendix.

	\section{Background}\label{sec:background}
	In this section, we provide a concise overview of key concepts in Dung-style argumentation theory, establish the basic terminology of computability theory, and introduce computable argumentation frameworks as a natural way to integrate these two domains. We assume the reader is acquainted with a few fundamental concepts from mathematical logic (such as ordinals and Turing degrees) as presented, e.g., in \cite{rogers1987theory}.
	
	\subsection{Argumentation theory background} \label{sec: back argumentation}
	
	An \emph{abstract argumentation framework} (AF) $\F$  is a pair $(A_\F,R_\F)$ consisting of  a set $A_\F$ of arguments and an attack relation $R_\F\subseteq A_\F\times A_\F$. If an argument $a$ attacks an argument $b$, we write $a\att b$ instead of $(a, b)\in R_\F$.   We call collections of arguments $S\subseteq A_\F$ \emph{extensions}. For an extension $S$, the symbols $S^+$ and $S^-$ denote, respectively, the arguments that $S$ attacks and the arguments that attack $S$: 
	$S^+=\{x : (\exists y \in  S)(y \att x)\}$ and
	$S^-=\{x : (\exists y \in  S)(x \att y) \}$.
	
	\begin{definition}\label{def:characteristic function}
		$S\subseteq A_{\F}$ \emph{defends} an argument $a$, if  any argument that attacks $a$ is attacked by some argument in $S$ (i.e., $\{a\}^-\subseteq S^+$). The \emph{defense function}  of $\F$ is the  following mapping $f_\F$ which sends subsets of $A_\F$ to subsets of $A_\F$: 
		\[
		f_\F(S) := \{x : \text{ $x$ is defended by $S$} \}.
		\]
	\end{definition}
	
	
	Dung \cite[Definition 20]{dung1995acceptability} introduced the grounded extension to capture the collection of arguments which should be accepted by a maximally skeptical reasoner. Dung defined the \emph{grounded extension} $\G$ as the least fixed-point of the defense function and noted that, for any AF, the grounded extension exists and is guaranteed to be conflict-free \cite[Theorem 25]{dung1995acceptability}, i.e., $a\natt b$ for every $a,b\in \G$. A fundamental property of the  grounded extension is that it can be approximated from below, by iterating along the ordinals the generative process sketched in the Introduction (see Definition \ref{defn:GroundingSets}).

	
	
	
	Recall that the ordinals, introduced to mathematics by Georg Cantor in 1883, extend the counting numbers into the infinite. We will use the following two properties of the collection of ordinals:
	
	\begin{theorem}\label{properties of ordinals}
		The collection of ordinals is linearly ordered, and every non-empty subset of the collection of ordinals has a least element.
	\end{theorem}
	
	Recall that the ordinal $\omega$ is the least infinite ordinal. Recall also that every ordinal $\alpha$ is either the successor of another ordinal $\beta$, i.e., $\alpha=\beta+1$, or is a limit ordinal, i.e., $\alpha = \sup(Y)$, which is the supremum of the ordinals in the set $Y$. The ordinal $\omega$ is the first limit ordinal, and is the limit of the finite ordinals, i.e., the natural numbers.
	
	

	
	
	
	\begin{definition}\label{defn:GroundingSets}
		For an AF $\F$ and all ordinals $\alpha$, let $\G_\alpha$ be inductively defined as follows:
		\begin{itemize}
			\item $\G_0=\emptyset$;
			\item If $\alpha=\beta+1$, then $\G_\alpha=f_{\F}(\G_{\beta})$;
			\item If $\alpha$ is a limit ordinal, then $\G_\alpha=\bigcup_{\beta < \alpha} \G_\beta$.
		\end{itemize}
	\end{definition}
	
	
	It follows from Dung's proof of the existence of the grounded extension that there is an ordinal $\beta$ so that $\G=\G_\alpha$ for every $\alpha\geq \beta$.
	
	\begin{definition}\label{def:grounding ordinal}
		For an AF $\F$, the least ordinal $\alpha$ so that $\G=\G_\alpha$ is the \emph{grounding ordinal} of $\F$.
	\end{definition}
	
	
	
	
	
	An initial bound on the size of the grounding ordinal for (countable) AFs can be established through a straightforward cardinality argument:
	
	\begin{observation}\label{grounding ordinals are countable}
		If $\F$ is a countable AF, then the grounding ordinal of $\F$ is countable.
	\end{observation}
	\begin{proof}
		Otherwise we would have that for each countable ordinal $\alpha$, $\G_{\alpha+1}\supsetneq \G_\alpha$. But there are uncountably many countable ordinals. Thus, $\bigcup_\alpha \G_\alpha\subseteq A_\F$ would be uncountable.
	\end{proof}

	Dung proved that, for \emph{finitary} AFs (i.e., those $\F$ so that the set $\{a\}^-$ is finite, for all $a\in A_\F$),  the grounding ordinal is $\leq \omega$.
	%
	Yet not every countable AF has a grounding ordinal $\leq \omega$. Baumann and Spanring \cite{BaumannSpanring} provide an example, illustrated in Figure \ref{BaumannSpanringFigure}, of an AF with a grounding ordinal of $\omega \cdot 2$. 
	
	\begin{figure}
		\centering
		
		\scalebox{0.8}{
			\begin{tikzpicture}
			
			\node[circle, draw] (a0) {$a_0$};
			\node[circle, draw] (a1) [right=of a0] {$a_1$};
			\node[circle, draw] (a2) [right=of a1] {$a_2$};
			\node[circle, draw] (a3) [right=of a2] {$a_3$};
			\node[circle, draw] (a4) [right=of a3] {$a_4$};
			\node[circle, draw] (a5) [right=of a4] {$a_5$};
			\node[draw=none] (ellipsis1)[right=1mm of a5] {$\cdots$};
			
			\node[circle, draw] (b0) 
			[below=of a0] {$b_0$};
			\node[circle, draw] (b1) [below=of a1] {$b_1$};
			\node[circle, draw] (b2) [below=of a2] {$b_2$};
			\node[circle, draw] (b3) [below=of a3] {$b_3$};
			\node[circle, draw] (b4) [below=of a4] {$b_4$};
			\node[circle, draw] (b5) [below=of a5] {$b_5$};
			\node[draw=none] (ellipsis2)[right=1mm of b5] {$\cdots$};

			\draw[->] (a0) edge (a1);
			\draw[->] (a1) edge (a2);
			\draw[->] (a2) edge (a3);
			\draw[->] (a3) edge (a4);
			\draw[->] (a4) edge (a5);
			
			\draw[->] (b0) edge (b1);
			\draw[->] (b1) edge (b2);
			\draw[->] (b2) edge (b3);
			\draw[->] (b3) edge (b4);
			\draw[->] (b4) edge (b5);

			\draw[->]  (a1) edge[out=-150, in=90] (b0);
			
			\draw[->]     (a3) edge[out=-140, in=60] (b0);
			
			\draw[->]    (a5) edge[out=-140, in=30] (b0);
			\end{tikzpicture}
		}

		\caption{Example from Baumann-Spanring \cite{BaumannSpanring} of an AF with grounding ordinal $\omega\cdot2$.} 
		\label{BaumannSpanringFigure}
	\end{figure}
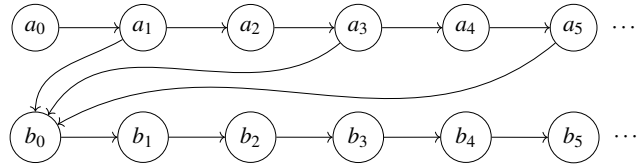

	In this paper, we determine exactly how large the grounding ordinal might be for a countable AF. Recall that $\omega_1$ is the least uncountable ordinal. We will prove (Corollary \ref{bound grounding ordinal for countable AFs}) that 
	\emph{an ordinal $\alpha$ is a grounding ordinal of a countable AF if and only if $\alpha$ is a countable ordinal.}
	
	
	
	In the next section, we introduce the computable argumentation frameworks and state a contrasting result to this, where we give a much better bound if $\F$ is computable.
	

	\subsection{Computability Background}
	
	
	Computable AFs are defined to be the AFs described by some computable function. We denote by  $(\Phi_e)_{e\in\nat}$  a uniform enumeration of all partial computable functions from $\nat\times \nat$ to $\{0,1\}$. The number $e$ simply encodes the finite set of instructions of the Turing machine which computes $\Phi_e$.
	
	\begin{definition}\label{indexing of computable AFs}
		A number $e$ is a \emph{computable index for an AF} $\F=(A_\F,R_\F)$ with $A_\F=\{a_n : n\in \nat\}$ if
		\[
		\Phi_e(n,m)=\begin{cases}
		1 &\text{ if $a_n \att a_m$}\\
		0 &\text{ otherwise}.
		\end{cases}
		\]
		An AF $\F$ is \emph{computable} if it has a computable index $e\in\nat$.
	\end{definition}

	Computable ordinals are defined in a similar way, as those ordinals which are given by computable functions:
	
	\begin{definition}
		A number $e$ is a \emph{computable index for a linear order} $L_e=\{a_n : n\in\nat \}$ if 
		\[
		\Phi_e( n,m)=
		\begin{cases}
		1 &\text{ if $a_n < a_m$}\\
		0 &\text{ otherwise}.
		\end{cases}
		\]
		
		If the linear order $L_e$ is isomorphic to the ordinal $\alpha$, then we say that $\alpha$ is a \emph{computable ordinal}.
	\end{definition}
	
	Observe that there are only countably many computable ordinals, and each of them is countable. However, there are uncountably many countable ordinals. Consequently, there exists a least ordinal that is not computable, but which is countable.
	
	\begin{notation}\label{definition: omega1CK}
		The least non-computable ordinal is commonly denoted by  $\CK$. 
	\end{notation}
	
	Our main results (Theorems \ref{make it big} and \ref{thm:CK upper bound}) on the grounding ordinals of computable AFs show that \emph{the grounding ordinals of computable AFs are exactly the ordinals $\leq \CK$.}
	

	We let $\nat^{<\nat}$ represent the collection of finite sequences of natural numbers and $\nat^\nat$ represent infinite sequences of natural numbers. 
	A string $\sigma\in \nat^{<\nat}$ is a \emph{prefix} of a string $\tau\in \nat^{<\nat}$ or of a sequence $\tau \in \nat^{\nat}$ if there is some $\rho\in \nat^{<\nat}$ or $\rho \in \nat^{\nat}$ so that $\tau = \sigma^\smallfrown\rho$, i.e, $\tau$ is the concatenation of $\sigma$ with $\rho$.
	
	If $\sigma$ is a prefix of $\tau$ and the length of $\tau$ is 1 more than the length of $\sigma$, then $\tau$ is an \emph{immediate extension} of $\sigma$, and we write $\sigma = \tau^-$.
	
	A \emph{tree} is a subset of $\nat^{<\nat}$ which is closed under prefixes. A path $\pi$ through a tree $T$ is an element of $\nat^\nat$ so that every prefix $\sigma$ of $\pi$ is in $T$. Note that since our trees may be infinitely branching, even trees which contain arbitrarily long strings may not have paths. Consider for example, the tree containing the empty string and all strings $i^\smallfrown\sigma$ where the length of $\sigma$ is $\leq i$. 

	For a tree $T$ which does not have a path, we define the \emph{rank} of $T$ to be an ordinal which captures how difficult it is to see that $T$ has no path.
	
	\begin{definition}
		For any $\sigma\in T$, if $\sigma$ is a terminal node in $T$ (i.e., it has no extensions in $T$), then $\sigma$ is assigned rank 0. If all immediate extensions of $\sigma$ are ranked, then $\sigma$ is ranked and has rank 
		\[
		\text{sup}(\{\rank(\tau)+1 : \tau \text{ is an immediate extension of }\sigma\}).
		\]
		This process gives a rank to all elements $\sigma$ so that $\sigma$ is not a prefix of a path in $T$. If the root of the tree is assigned a rank $\alpha$ (i.e., there is no path in $T$), then we say the tree itself has rank $\alpha$.
	\end{definition}
	
	
	The following bounding theorems connect the computable ordinals to the ranks of computable trees:
	
	\begin{theorem}[see \protect{\cite[\S 16.3, Theorem XIX(a)]{rogers1987theory}}]\label{Earlier Bounding Theorem}
		If a computably enumerable tree has no path, then the rank of the tree is an ordinal $<\CK$.
	\end{theorem}
	
	\begin{theorem}[see \protect{\cite[\S 16.3, Theorem XIX(b)]{rogers1987theory}}]\label{thm:high ranks for trees}
		If $\alpha$ is a computable ordinal, then there exists a computable tree with rank $\alpha$.
	\end{theorem}
	
	
	
	Next, we shift to the background needed for understanding our complexity result of $\Pi^1_1$-completeness of $\G$ for a computable AF.
	
	\begin{definition}
		The $\Sigma^1_1$ sets are formally defined as those subsets of $\nat$ that are definable in the language of second-order arithmetic using a single second-order existential quantifier ranging over subsets of $\nat$ followed by number quantifiers 
		and the first order functions and relations $(+,\cdot,<,0,1,\in)$; for more details, see \cite[\S 16]{rogers1987theory}. $\Pi^1_1$ sets are the complements of $\Sigma^1_1$ sets.
	\end{definition}
	
	Note that the $\Sigma^1_1$ sets are the analog in the infinite setting of the NP problems in the finite setting. Both are described by an exhaustive search over all subsets of the domain. Similarly, the $\Pi^1_1$ sets are the analog of the co-NP problems in the finite setting.
	
	\begin{observation}\label{obs:grounded extensions are Pi11}
		The set of pairs $(a,e)$, where the argument $a$ belongs to the grounded extension of the AF with computable index $e$, is a $\Pi^1_1$ set. This follows from the fact that the grounded extension is the \emph{smallest}  fixed-point of the defense function $f_\F$. Specifically, $\G$ is defined by the formula $\forall X (f_{\F}(X)=X \rightarrow a\in X)$. Note that the defense function is defined in terms of the attack relation, which is computable, and quantifiers over arguments, which are encoded by natural numbers. Therefore, the only second-order quantifier involved is the initial universal quantifier $\forall X$.
	\end{observation}
	
	
	The following theorem states that the set of computable indices for trees which have paths is a maximally complex $\Sigma^1_1$ set.
	
	\begin{theorem}[Kleene \cite{kleene1955arithmetical}] \label{characterization analytic sets} \label{Kleene}
		A set $X\subseteq \nat$ is $\Sigma^1_1$ if and only if there is a computable sequence of computable trees $(\T^X_n)_{n\in\nat}$ so that $n\in X$ iff $\T^X_n$ has a path.
	\end{theorem}

	\begin{definition}\label{def:Pi11 completeness}
		A set $X$ is \emph{$\Pi^1_1$-complete} if it is $\Pi^1_1$ and whenever $Y$ is a $\Pi^1_1$ set, there exists a computable function $f:\nat\rightarrow\nat$ so that $n\in Y$ if and only if $f(n)\in X$.
	\end{definition}

	Finally, all of these notions relativize. For example, for any Turing degree $d$, we say that an AF is $d$-computable, if there is a $d$-computable function $\Phi$ giving the attack relation. We say a set is $\Pi^1_1(d)$ if it can be defined in a $\Pi^1_1$ way using a predicate for an oracle in $d$. All the theorems in this section are valid in their relativized form.

	

	\section{AFs with maximally large grounding ordinals and $\Pi^1_1$-completeness}\label{sec:maximially long and hard}
	
	In this section, we begin with an example to see how we can build computable AFs with large grounding ordinals. Then we use this example to give computable AFs with any possible grounding ordinal, namely each $\alpha\leq \CK$. 

	\begin{example}\label{exampleAnyComputableOrdinal}
		Let $\T$ be a tree. Let $\F_\T$ 
		be defined by having elements $\{a_\sigma : \sigma \in \T\}\cup \{b_\sigma : \sigma\in \T\}$. We let $b_\sigma \att a_{\sigma^-}$ where $\sigma$ is an immediate extension of $\sigma^-$ and $a_\sigma\att b_\sigma$. Note that $\F_\T$ is a computable AF if $\T$ is a computable tree.
	\end{example}
	
	\begin{restatable}{theorem}{CharacterizeGsInExample}
		\label{thm:characterizeGsInExample}
		In $\F_\T$, $\G_{\beta+1}$ is exactly the set of $a_\sigma$ so that $\sigma$ has rank $\leq \beta$.
	\end{restatable}

	We deduce the following theorem by applying Theorem \ref{thm:characterizeGsInExample} to all computable ordinals:
	
	\begin{restatable}{theorem}{HittingAnySuccessor}
		\label{HittingAnySuccessor}
		If $\alpha=\beta+1$ is a computable successor ordinal, then there exists a computable AF $\F_\alpha$ with grounding ordinal $\alpha$. 
		
	\end{restatable}

	By taking an appropriate union of the AFs from Example \ref{exampleAnyComputableOrdinal}, we can also build AFs whose grounding ordinal is a computable limit ordinal.
	
	\begin{restatable}{theorem}{HittingLimitOrdinals}\label{HittingLimitOrdinals}
		If $\alpha$ is any computable ordinal, then there exists a computable AF $\F_\alpha$ with grounding ordinal $\alpha$. 
	\end{restatable}
	
	
	We now proceed to construct a computable AF whose grounding ordinal is $\CK$. 
	

	\begin{theorem}\label{Can be Complete}\label{make it big}
		There are computable AFs where the grounding ordinal is $\CK$.
		There are computable AFs where $\G$ is $\Pi^1_1$-complete. 
	\end{theorem}
	\begin{proof}
		We construct the AF $\F$ in the following Lemma:
		\begin{restatable}{lem}{allTreesTogether}\label{allTreesTogether}
			There is a computable AF $\F$ isomorphic to the disjoint union of $\{
			\F_\T : \T$  a computably enumerable tree$\}$.
		\end{restatable}
		Note that an argument $a\in A_{\F_{\T}}$ is in $\G_\alpha$ in the disjoint union $\F$ if and only if it is in $\G_\alpha$ in $\F_{\T}$. Thus for the root $\lambda_\T$ of the tree $\T$, the argument $a_{\lambda_\T}$ is in $\G$ if and only if $\T$ has no path. It follows from Theorem \ref{Kleene} that $\G$ is a $\Pi^1_1$-complete set. 
		
		Since every computable tree is among the $\T$ included in $\F$, and there are trees with arbitrarily high rank among the computable ordinals by Theorem \ref{thm:high ranks for trees}, we cannot have $\G=\G_\alpha$ for any $\alpha<\CK$. Thus the grounding ordinal is $\geq \CK$.
		On the other hand, every element of $\G$ is an $a_\sigma\in \F_{\T}$ where $\sigma$ is ranked in $\T$ by Theorem \ref{thm:characterizeGsInExample}. Then the set of strings $S=\{\tau : \sigma^\smallfrown \tau\in \T\}$ forms a ranked computably enumerable tree with the same rank as the rank of $\sigma$ in $\T$. Thus the rank of $\sigma$ in $\T$ is some $\beta <\CK$ by Theorem \ref{Earlier Bounding Theorem}, so $a_\sigma\in \G_{\beta+1}\subseteq \G_{\CK}$, which shows that $\G=\G_\CK$.
	\end{proof}
	
	
	Having established the existence of computable AFs with grounding ordinal $\alpha$ for any $\alpha \leq \CK$, we now turn our attention to proving that $\CK$ is the largest possible grounding ordinal for a computable AF. 
	
	\section{Upper bounds on grounding ordinals}\label{sec:upper bounds}
	From Observation \ref{obs:grounded extensions are Pi11}, it follows that the grounded extension of any computable AF $\F$ is a $\Pi^1_1$-set. Hence, Theorem \ref{Kleene} ensures that there is a sequence of trees $\T^a$, for $a\in A_\F$, such that $a\in \G$ if and only if $\T^a$ has no path. In order to yield a restriction on the grounding ordinal of an AF, we need to construct these trees in a way that the rank of $\T^a$ corresponds to the least $\alpha$ such that $a\in \G_\alpha$. To do so, we highlight a natural class of extensions in AFs which we call ``self-defending'' and use these to give a dual understanding of the grounded extension in terms of the largest self-defending extension instead of the least fixed-point of $f_\F$. 
	This construction is detailed in Section \ref{sec:tree rank analysis of groundedness}.
	Once established, we will apply Theorem \ref{Earlier Bounding Theorem} to conclude that each $a\in \G$ belongs to some $\G_\alpha$ with $\alpha<\CK$, thereby showing that $\G_{\CK}=\G$.  

	\subsection{Rank analysis for groundedness}\label{sec:tree rank analysis of groundedness}
	
	In order to show that an element $a$ is in $\G$, we endeavor to show that every $b$ which attacks $a$ is in $\G^+$. That is, we show that $a\in f_{\F}(\G)$, and since $\G=f_{\F}(\G)$, this shows that $a\in \G$. Thus, we first build a tree $T_{\{b\}}$ so that $b\in \G^+$ if and only if $T_{\{b\}}$ has no path. We understand $\G^+$ in terms of the following property of $\G^+$.
	
	\begin{definition}
		A set $X\subseteq A_{\F}$ is \badadmissible if whenever $y\in A_\F$ attacks some element of $X$, then there is some $x\in X$ so that $x\att y$. (Note that we are \emph{not} assuming that $X$ is conflict-free.)
	\end{definition}
	
	To the best of our knowledge, this result has been previously established, though we do not know an appropriate citation.
	
	\begin{restatable}{lemma}{LargestSelfDefending}
		\label{lem:largest badadmissible extension}
		The largest \badadmissible extension is the complement of $\G^+$.
	\end{restatable}
	
	
	This suggests a strategy for showing,  in a $\Sigma^1_1$ way, that an element is in $A_\F\smallsetminus \G^+$: Show that it is contained in a \badadmissible extension. 
	For any finite set $S$, we produce a tree so that the paths through $S$ encode \badadmissible extensions containing $S$.
	
	Recall that we fix an enumeration of the arguments of $\F$ as $A_{\F}=\{a_n : n\in \nat\}$. 
	
	\begin{definition}
		Let $S$ be any finite subset of $A_\F$.
		
		For $\sigma\in \nat^{<\nat}$, define $\mran_S(\sigma)=S\cup \{a_{\sigma(m)-1} : \sigma(m)\geq 1\}$.
		
		For $\pi\in \nat^\nat$, define $\mran(\pi) = \bigcup_{\sigma\preceq\pi} \mran(\sigma)$.
	\end{definition}
	
	We use a computable function $x,y\mapsto \langle x,y \rangle$ which is a bijection from $\nat\times \nat$ to $\nat$.
	
	\begin{definition}\label{defTreeForSelfDefending}
		For $S$ any finite subset of $A_\F$, we define a tree $T_S$ as follows:
		\begin{itemize}
			\item The empty string is in $T_S$.
			\item Let $\sigma\in T_S$ with $|\sigma| = \langle n ,m \rangle$.
			\begin{itemize}
				\item  Case 1: $a_n$ attacks some element of $\mran_S(\sigma)$. 
				Then $\sigma^\smallfrown i\in T_S$ if and only if $i>0$ and $a_{i-1}\att a_n$;
				\item Case 2: $a_n$ does not attack any element of $\mran_S(\sigma)$. Then $\sigma^\smallfrown i\in T_S$ if and only if $i=0$.
			\end{itemize}
		\end{itemize}
	\end{definition}
	
	Informally, each path $\pi$ describes the set $\mran(\pi)$. At each level of the tree, we consider an element $a_n$ and for any $\sigma$ at that level of the tree, the immediate extension of $\sigma$ must justify how the set $\mran(\pi)$ is self-defending in the face of a possible attack by $a_n$. The extension $\sigma^\smallfrown 0$ represents the justification that there is no attack from $a_n$ to any element of $\mran(\sigma)$, and the extension $\sigma^\smallfrown (i+1)$ represents the justification that $a_i\in \mran(\pi)$ defends against an attack from $a_n$.
	
	\begin{restatable}{lemma}{SelfDefendingIFFPath}
		\label{lem:badadmissible extension iff path in TS}
		$S$ is contained in a \badadmissible extension if and only if $T_S$ has a path.
	\end{restatable}
	
	
	\begin{lemma}\label{disjoint from G+ iff have path}
		$S\cap \G^+=\emptyset$ if and only if $T_S$ has a path.
	\end{lemma}
	\begin{proof}
		This follows directly from Lemmas \ref{lem:largest badadmissible extension} and \ref{lem:badadmissible extension iff path in TS}.
		%
	\end{proof}
	
	The following Lemma is the property of these trees which bridges between tree-rank and grounding ordinals.
	
	\begin{lemma}\label{lem:ranks of TS}
		For any $\rho\in T_S$, if the rank of $\rho$ is $\leq \alpha$, then there is a member of $\G_{\alpha+1}$ which attacks some element of $\mran_S(\rho)$.
	\end{lemma}
	\begin{proof}
		We prove this by induction on $\alpha$. For $\alpha=0$, let $\rho$ be given and let $|\rho|=\langle n,m\rangle$. Then $\rho\in T_S$ having rank $0$ means it is a terminal node, so $a_n$ attacks some member of $\mran_S(\rho)$ and no element attacks $a_n$, i.e., $a_n\in \G_1$.
		Let $\alpha>0$ be fixed and fix $\rho\in T_S$ with rank $\leq \alpha$. Let $|\rho|=\langle n,m\rangle$.
		If $a_n$ does not attack any element of $\mran_S(\rho)$, then the tree above $\rho$ has rank 1 more than the rank of the tree above $\rho^\smallfrown 0$, i.e., $\alpha=\rank(\rho^\smallfrown 0)+1$. By inductive hypothesis applied to $T_S, \rho^\smallfrown 0,\rank(\rho^\smallfrown 0)$, some element of $\G_{\alpha}$ attacks some element of $\mran_S(\rho^\smallfrown 0)=\mran_S(\rho)$.
		
		Next, suppose that $a_n$ does attack an element of $\mran_S(\rho)$.
		If $a_n\in \G_{\alpha+1}$, we are done, so we suppose $a_n\notin \G_{\alpha+1}$.
		Then there must be some $i$ so that $a_i\att a_n$ and $a_i\notin \G_\alpha^+$. By inductive hypothesis, some member of $\mran_S(\rho^\smallfrown (i+1))=\mran_S(\rho)\cup \{a_i\}$ is attacked by a member of $\G_\alpha$. Since this cannot be $a_i$, we have that a member of $\G_\alpha\subseteq \G_{\alpha+1}$ attacks a member of $\mran_S(\rho)$. 
	\end{proof}
	
	Applying this Lemma to the empty string, we get the following corollary:
	
	\begin{corollary}
		If the rank of $T_{\{b\}}$ is less than $\alpha$, then $b$ is attacked by an element of $\G_{\alpha+1}$.
	\end{corollary}

	We can thus use the trees $T_{\{b\}}$ for each $b$ attacking $a$ to analyze whether or not $a$ has an attacker which is not in $\G^+$. Putting all  these trees together, we define a tree $\T^a$ as follows:
	
	\begin{definition}
		For each $a\in A_\F$, we define the tree $\T^a$ to be the tree which puts the empty string $\lambda\in \T^a$ and $i^\smallfrown \sigma \in \T^a$ if and only if $a_i\att a$ and $\sigma\in T_{\{a_i\}}$.
	\end{definition}
	
	\begin{corollary}\label{cor:Ta has a path iff a not in G}
		$\T^a$ has a path if and only if $a\notin \G$.
	\end{corollary}
	\begin{proof}
		$\T^a$ has a path if and only if there is some $a_i$ with $a_i\att a$ and so $\T_{\{a_i\}}$ has a path if and only if there is some $a_i\att a$ so $a_i\notin \G^+$, if and only if $a$ is not defended by $\G$ if and only if $a\notin\G$.
	\end{proof}
	
	
	Finally, we connect the rank of the tree $\T^a$ to the ordinal $\alpha$ so that $a\in \G_\alpha$.
	
	\begin{corollary}\label{uniform bounding wins}
		If $\T^a$ has rank $\leq \alpha$, then $a\in \G_{\alpha+1}$.
	\end{corollary}
	\begin{proof}
		If $\T^a$ has rank $\leq \alpha$ then every $a_i$ which attacks $a$ has $\text{rank}(T_{\{a_i\}}) <\alpha$. Thus, each element of $\A_\F$ which attacks $a$ is itself attacked by a member of $\G_\alpha$, and therefore $a\in \G_{\alpha+1}$.
	\end{proof}

	\subsection{Concluding Upper Bounds}\label{sec:main results}

	\begin{theorem}\label{thm:CK upper bound} Let $\F$ be a computable AF. Then the grounding ordinal is $\leq \CK$. That is, the grounded extension equals $\G_{\CK}$, i.e., $\G=\G_{\CK}$.
	\end{theorem}
	\begin{proof}
		Suppose that $a\in \G$. Then the computable tree $\T^a$ has no path. By Theorem \ref{Earlier Bounding Theorem},
		the rank $\alpha$ of $\T^a$ is $<\CK$. It follows from Corollary \ref{uniform bounding wins} that $a\in \G_{\alpha+1}$, thus $a\in \G_{\CK}$.
	\end{proof}

	
	
	
	\section{Conclusions beyond computable AFs}\label{sec:conclusions for non-computable AFs}
	
	Though our analysis was focused on computable AFs, everything relativizes to any Turing degree:
	
	\begin{theorem}\label{Can be Complete-d}\label{make it big-d}
		Let $d$ be a Turing degree. There are $d$-computable AFs where the grounding ordinal is $\CKD$.
		There are $d$-computable AFs where $\G$ is $\Pi^1_1(d)$-complete. 
	\end{theorem}
	\begin{proof}
		We use the same construction as in Theorem \ref{Can be Complete} relativized to $d$. We take a $d$-computable sequence of all $d$-computable trees $\T_i$ and construct a $d$-computable AF $\F$ so that $\G$ is $\Pi^1_1(d)$-complete. It follows from the relativized version of Theorem \ref{thm:high ranks for trees} that the grounding ordinal of $\F$ is $\geq \CKD$.
		It follows from the relativized version of Theorem \ref{thm:CK upper bound} that the grounding ordinal of $\F$ is $\leq \CKD$.
	\end{proof}

	\begin{theorem}\label{thm: arbitrarily large countable grounding ordinals}
		For any countable ordinal $\alpha$, there is a countable AF $\F$ such that the grounding ordinal for $\F$ is $\alpha$.
	\end{theorem}
	\begin{proof}
		For any countable ordinal $\alpha$, there is an ordering on $\nat$ which is isomorphic to $\alpha$. Let $d$ be the Turing degree of this ordering. Then $\alpha$ is $d$-computable. By Theorem \ref{HittingLimitOrdinals} relativized to $d$, there is a $d$-computable AF $\F$ so that the grounding ordinal is $\alpha$.
	\end{proof}
	
	The following Corollary is immediate from Observation \ref{grounding ordinals are countable} and Theorem \ref{thm: arbitrarily large countable grounding ordinals}.
	
	\begin{corollary}\label{bound grounding ordinal for countable AFs}
		An ordinal $\alpha$ is a grounding ordinal of a countable AF if and only if $\alpha$ is countable.
	\end{corollary}
	
	Since our analysis is constructive and we have strong bounds on the levels in the analysis, we are also able to extend the work begun by Spanring \cite{spanringThesis} and show that not only is the Axiom of Choice unnecessary to prove the existence and uniqueness of the grounded extension, but the very weak base theory of Kripke–Platek set theory (KP) suffices (see \cite{barwiseBook} for a reference on KP). \footnote[5]{We thank Noah Schweber for providing the proof of Theorem \ref{KP suffices}.}
	
	\begin{restatable}{theorem}{KPSuffices}
		\label{KP suffices}
		There exists a formula $\psi(x,y)$ in the language of set theory so that, for any AF $\F$ in a model of KP, $\psi(x,\F)$ defines the grounded extension of $\F$. Further, KP proves that for every AF $\F$, $\psi(x,\F)$ defines a fixed point of $f_{\F}$ in $\F$. Further, KP proves for each $\rho(x,y)$ that for every $Y$, if $\rho(x,Y)$ defines a fixed point of $f_{\F}$ in $\F$, then $\forall x (\psi(x,\F)\rightarrow \rho(x,Y))$; that is, among all classes, $\psi$ is the least fixed-point of $f_{\F}$ in $\F$. 
	\end{restatable}

	\section{The Turing degrees of grounded extensions}\label{sec:Turing degrees missing}
	
	We finally turn to the question of determining the exact complexity of the membership problem for the grounded extension of a given computable AF $\F$. Since $\G$ is a $\Pi^1_1$ set, and  Theorem \ref{make it big} demonstrated that $\G$ can be maximally complex among the $\Pi^1_1$-sets, it is natural to conjecture that the Turing degree of \emph{any} $\Pi^1_1$ set could correspond to the complexity of $\G$ for some computable $\F$. It is a common phenomenon in computability theory that once the complexity of natural class $\C$ of sets is identified, this gives a characterization of the Turing degrees of members of $\C$. For example, the Turing degrees of word problems of finitely presented groups are exactly the Turing degrees of semicompuable sets \cite{Groups1,Groups2,Groups3}.
	
	However, we show that this is not the case for grounded extensions. We show this by identifying a new notion of a set being \emph{\uswp} which we show holds for grounded extensions, showing that they have more structural uniformity than a general $\Pi^1_1$ set. 
	This notion is of independent interest in computability theory.
	
	
	
	
	\begin{defn}
		A set $X$ is \emph{\uswp} if there is a computable sequence of computable trees $(\T_i)_{i\in \nat}$ and a Turing functional $\Psi(A,n)$ so that $\T_i$ has no path if and only if $i\in X$. Further, whenever $i\notin X$, then $\Psi(X',i)$ is a path in $\T_i$. 
	\end{defn}

	%

	\begin{restatable}{lem}{groundedExtsareUSWP}
		\label{groundedExts are USWP}
		Let $\F$ be a computable argumentation framework. Then the grounded extension is \uswp.
	\end{restatable}

	
	
	The proof of the following Lemma requires several ideas from computability, including relativized hyperimmune-freeness and Friedberg jump-inversion.
	
	\begin{restatable}{lemma}{PIINOTUSWP}
		\label{There are Pi11 degrees which do not contain a USWP}
		There are Turing degrees of $\Pi^1_1$ sets which do not contain a set which is \uswp.
	\end{restatable}
	These two Lemmas yield the following theorem.
	\begin{theorem}
		
		There are $\Pi^1_1$-sets $X$ so that the Turing degree of $X$ is not the Turing degree of the grounded extension of any computable argumentation framework.
	\end{theorem}
	
	A natural direction for future research is to characterize exactly which Turing degrees contain grounded extensions of computable AFs. Is it the degrees of sets which are \uswp?

	\bibliographystyle{eptcs}
	\bibliography{biblio}

	\appendix
	\section{Additional details for Section \ref{sec:maximially long and hard}}
	
	\CharacterizeGsInExample*
	\begin{proof}
		We check by an induction on ordinals $\beta$ that $\G_{\beta+1}$ is exactly the set of $a_\sigma$ so that $\sigma$ has rank $\leq \beta$. For $\beta=0$, we note that $\G_1$ is the unattacked elements of $\F_\T$, i.e., the $a_\sigma$ where $\sigma$ is a terminal node on $\T$, meaning $\sigma$ has no immediate extensions in $\T$. These are exactly the $a_\sigma$ where $\sigma$ has rank $\leq 0$.
		
		For successor $\beta=\gamma+1$, $a_\sigma\in \G_{\beta+1}=\G_{\gamma+2}$ if and only if $a_\tau\in \G_{\gamma+1}$ for every immediate extension $\tau$ of $\sigma$. This, by inductive hypothesis is exactly if every $\tau$ an immediate successor of $\sigma$ has rank $\leq \gamma$, which is equivalent to the rank of $\sigma$ being $\leq \beta=\gamma+1$.
		
		
		Finally, if $\beta$ is a limit, then $a_\sigma\in \G_{\beta+1}$ if and only if every $a_\tau$ for $\tau$ an immediate extension of $\sigma$ is in $\G_{\beta}$. Since $\beta$ is a limit, this means every $a_\tau$ for $\tau$ an immediate extension of $\sigma$ is in some $\G_{\delta_\tau+1}$ for some $\delta_\tau <\beta$. But this is equivalent by the inductive hypothesis to the rank of $\tau$ being $\leq \delta_\tau<\beta$ for every $\tau$ an immediate extension of $\sigma$. But this is equivalent to the rank of $\sigma$ being $\leq \beta$. 
	\end{proof}
	
	\HittingAnySuccessor*
	\begin{proof}
		By Theorem \ref{thm:high ranks for trees}, we can take a tree $\T$ with rank $\beta$. Then the grounded extension in $\F_\T$ in Example \ref{exampleAnyComputableOrdinal} is the set of all $a_\sigma$ since every $\sigma\in \T$ is ranked. Further $\G=\G_\alpha$ since every element of $\T$ has rank $\leq \beta$ and $\alpha$ is least such.  
	\end{proof}
	
	\HittingLimitOrdinals*
	\begin{proof}
		If $\alpha$ is a successor ordinal, then the result is established in Theorem \ref{HittingAnySuccessor}, so we suppose that $\alpha$ is a limit ordinal. Let $e$ be a computable index for a linear order $L_e$ isomorphic with $\alpha$. Then by considering the sequence of sets $S_a:=\{x\in L_e : x < a_n\}$ for each $n\in \nat$, we get a uniformly computable sequence of indices $(i_n)_{n\in \omega}$ for computable ordinals $\beta_n$ so that $\alpha = \sup (\{\beta_n : n\in \nat\})$. Using Theorem \ref{thm:high ranks for trees}, we can get a computable sequence $\T_n$ of computable trees so that the rank of $\T_n$ is $\beta_n$. Applying Example \ref{exampleAnyComputableOrdinal}, we get a uniformly computable sequence $\F_{\T_n}$ of AFs with grounding ordinal $\beta_n+1$. Finally, we let $\F$ be the disjoint union of the AFs $\F_{\T_n}$. Then the grounding ordinal of $\F$ is the supremum of $\{\beta_n+1 : n\in \nat\}$, which is $\alpha$ since each $\beta_n$ is $<\alpha$ and $\alpha$ is a limit ordinal. 
	\end{proof}
	
	\allTreesTogether*
	\begin{proof}
		We begin by indexing the computably enumerable (c.e.) presentations of all trees. For each index \( e \), define the tree \( \mathcal{T}_e \) such that \( \sigma \in \mathcal{T}_e \) if and only if \( \phi_e(\sigma)\downarrow = 1 \) for every \( \tau \preceq \sigma \). Note that every computably enumerable tree can be represented as some \( \mathcal{T}_e \), and each \( \mathcal{T}_e \) is indeed a c.e.\ tree.
		
		We construct an argumentation framework \( \mathcal{H} = (A_{\mathcal{H}}, R_{\mathcal{H}}) \) as follows
		$A_{\mathcal{H}} = \{(x, i) : x \in \mathcal{F}_{\mathcal{T}_i}\}$,
		where \( \mathcal{F}_{\mathcal{T}_i} \) denotes the argumentation framework derived from tree \( \mathcal{T}_i \) as in Example \ref{exampleAnyComputableOrdinal}. The attack relation is defined by
		$(x, i) \rightarrow (y, j) \quad \text{if and only if} \quad i = j \text{ and } x \rightarrow y \text{ in } \mathcal{F}_{\mathcal{T}_i}.$
		
		This framework \( \mathcal{H} \) is almost computable: its domain is a computably enumerable set, though not necessarily a computable one. To obtain a fully computable argumentation framework, observe that for any infinite c.e.\ set \( X \), there exists a computable bijection \( \beta : \omega \to X \). We use this to define a new framework \( \mathcal{F} = (A_{\mathcal{F}}, R_{\mathcal{F}}) \), where
		$A_{\mathcal{F}} = \{ a_i : i \in \mathbb{N} \}$,
		and \( a_x \rightarrow a_y \) in \( \mathcal{F} \) if and only if \( \beta(x) \rightarrow \beta(y) \) in \( \mathcal{H} \).
	\end{proof}
	
	\section{Additional details for Section \ref{sec:upper bounds}}
	
	\LargestSelfDefending*
	\begin{proof}
		We first show that $\G^+ \cap X=\emptyset$ whenever $X$ is \badadmissible. Suppose otherwise that $X$ is \badadmissible and some element of $\G$ attacks an element of $X$. Let $\alpha$ be least so that some $a\in \G_\alpha$ attacks an element of $X$. Then since $X$ is \badadmissible, some element $x\in X$ attacks $a$. But since $a\in \G_\alpha$, some element of $\G_\beta$ for $\beta<\alpha$ must attack $x$ contradicting the minimality of $\alpha$. It follows that any \badadmissible extension is contained in the complement of $\G^+$.
		
		Now we show that the complement of $\G^+$ is \badadmissible. If $a$ attacks $x\in A_{\F}\smallsetminus \G^+$, then $a\notin \G$. Thus $a\notin f_{\F}(\G)$. Thus there is some element $b\att a$ so that $b$ is not in $\G^+$, so $b\in A_{\F}\smallsetminus \G^+$ showing that $A_{\F}\smallsetminus \G^+$ is \badadmissible. 
	\end{proof}
	
	\SelfDefendingIFFPath*
	\begin{proof}
		Suppose that $T_S$ has a path $\pi$. Let $X=\mran_S(\pi)$. We show that $X$ is a \badadmissible extension. If $a_n$ attacks an element of $\mran_S(\pi)$, then for some sufficiently large $k$, $\pi(\langle n,k \rangle)>0$ and $a_{\pi(\langle n,k \rangle)-1}\in \mran_S(\pi)$ attacks $a_n$. Thus $X$ is \badadmissible.
		
		Suppose that $S$ is contained in a \badadmissible extension $X$. Then we can define a path $\pi$ in $T_S$ by extending finite strings $\sigma$ according to the following rule: Whenever $\vert \sigma\vert=\langle n,m\rangle$ and $a_n$ attacks some element of $\mran_S(\sigma)$, extend $\sigma$ by $i+1$ for some $i$ with $a_i\in X$ and $a_i\att a_n$. Since $X$ is \badadmissible, we can continue as such infinitely giving a path $\pi$ in $T_S$.
	\end{proof}

	\section{Additional details for Section \ref{sec:conclusions for non-computable AFs}}
	
	We note that the proof of the following theorem uses the machinery of set-theoretic forcing. Loosely speaking, we understand truth in a given model $A$ of KP-set theory by constructing a related larger model $A[G]$ with a convenient property, here that a given AF $\F$ is countable in $A[G]$, and then we use the relationship between $A$ and $A[G]$ to transfer the definition of the grounded extension of $\F$ from $A[G]$ to $A$.
	
	\KPSuffices*
	\begin{proof}
		We first show that there is a formula $\psi(x,y)$ so that whenever $\F$ is an AF in a model $A$ of KP, then $\psi(x,\F)$ defines the grounded extension of $\F$.
		
		We note that our construction of $\T^a$ above requires listing off the elements of $\F$. Thus we will first give a formula $\psi_c$ so that $\psi_c(x,\F)$ defines the grounded extension for any countable $\F$ in a model $A$ of KP. The formula $\psi_c(x,\F)$ says that there exists an ordinal $\alpha$ which is a rank for the tree $\T^x$ as constructed above. Since the construction of $\T^x$ above is computable, this is uniformly defined in the model $A$.
		If $A$ is a model of KP containing $\F$ and $a$ is in the grounded extension of $\F$, then there is an $\F$-computable ordinal $\alpha$ so that $\alpha$ is the rank of $\T^a$. Since $A\models \text{KP}$, this ordinal $\alpha$ is in $A$, so the formula $\psi_c(a,\F)$ will hold.
		On the other hand, if $a$ is not in the grounded extension of $\F$, then $\T^a$ has a path, so no ordinal can be a rank for $\T^a$. Thus $\psi_c(x,\F)$ defines the grounded extension $\F$ in $A$ whenever $\F$ is countable and $A$ models KP. 
		
		
		We now use a forcing argument to extend this to all AFs $\F$ in any model $A$ of KP:
		Suppose $A$ is a model of KP with an infinite set and let $\F$ be an AF in $A$. Then the forcing partial order $Col(\omega,A_\F)$ is an element of $A$, and it is known (see Barwise \cite[Final Lemma in Appendix B]{barwise}) that set forcing preserves admissibility, i.e., the forced model $A[G]$ also models KP, and satisfies the definability theorem in the following sense:
		
		\begin{enumerate}
			\item If $G$ is a filter through $Col(\omega,X)$ meeting every definable-over-$A$ dense subset of $Col(\omega,X)$, then $A[G]$ is a model of KP and $\F$ is countable in $A[G]$. Call such $G$ ``strongly generic over $A$."
			\item Moreover, for each fixed formula $\varphi$ the relation $p\Vdash \varphi(\nu)$, where $p$ ranges over conditions and $\nu$ ranges over $A$-names, is definable in $A$. (Here we read ``$p\Vdash\theta$" as ``For every strongly generic $G$ extending $p$, we have $A[G]\models\theta$.")
		\end{enumerate}
		
		We can now reason as follows. 
		Since $\psi_c(x,\F)$ defines $\G$ whenever $A$ models KP and $\F$ is a countable AF in $A$, \emph{for every generic $G$}, $\psi_c(x,\F)$ defines exactly $\G$ in $A[G]$. Thus, for every $a\in \G$, $\emptyset\Vdash \psi_c(a,\F)$. Similarly, for every $a\notin \G$, $\emptyset\Vdash \neg\psi_c(a,\F)$. Let $\psi(a,\F)$ be the formula whose existence is guaranteed in (2) which defines the condition  $\emptyset\Vdash \psi_c(a,\F)$. Now $\psi(x,\F)$ defines $\G$ for every $A$ modeling KP and every AF $\F$ in $A$.
		
		Now we shift to showing that KP proves that $\psi(x,\F)$ defines a fixed-point of $f_{\F}$ and that for every other definable fixed-point $\rho(x,Y)$, KP proves that $\rho(x,Y)$ contains the set defined by $\psi(x,\F)$. 
		This argument is an application of the completeness theorem for first-order logic. Since we have shown that $\psi(x,\F)$ defines the grounded extension $\G$ in any model of KP, the completeness theorem for first-order logic implies that every first order consequence of this fact is provable in KP. In particular, KP proves that $\psi(x,\F)$ defines a fixed point of $f_{\F}$. Similarly, whenever $\rho(x,Y)$ defines a fixed-point of $f_{\F}$ in a model of KP, it necessarily contains $\G$ since $\G$ is the least fixed-point of $f_{\F}$, and $\G$ is defined by $\psi(x,\F)$. Thus, every model of KP satisfies the statement that whenever $\rho(x,Y)$ defines a fixed-point of $f_{\F}$, then $\rho(x,Y)$ contains $\psi(x,\F)$, and the completeness theorem again shows that KP proves this first-order fact.
	\end{proof}

	\section{Additional details for Section \ref{sec:Turing degrees missing}}
	
	\groundedExtsareUSWP*
	\begin{proof}
		For each $a$, we have the computable tree $\T^a$ so that $a$ is in the grounded extension if and only if $\T^a$ has no path.
		Note that if $a$ is not in the grounded extension, then $\T^a$ has a path computed by $\G^+$. Namely, the path extends $i$ where $i$ is least so that $a_i\notin \G^+$. As in Lemma \ref{lem:badadmissible extension iff path in TS}, we can compute a path through $T_{\{a_i\}}$ by successively either extending by $0$ when in case 2 of Definition \ref{defTreeForSelfDefending} or by $i+1$ where $i$ is least so that $a_i\att a_n$ and $a_i\notin \G^+$. This is uniformly computable from $\G^+$, which is computable from $\G'$.
	\end{proof}
	
	We finally include the proof for Lemma \ref{There are Pi11 degrees which do not contain a USWP}. To do so, we introduce two additional concepts from computability theory: The hyperimmune-free degrees (relativized) and Friedberg jump inversion.
	
	\begin{definition}
		For a fixed Turing degree $d$, a Turing degree $c$ is said to be $d$-hyperimmune-free if whenever $f:\nat\rightarrow \nat$ is a function which is computable from $c$, there exists a $d$-computable function $g:\nat\rightarrow \nat$ so that $\forall x (f(x)\leq g(x))$.
	\end{definition}
	
	The notion of $d$-hyperimmune-freeness captures the notion that, though $c$ may compute non-$d$-computable functions, none of these have faster growth-rates than $d$-computable functions. The following theorem is proven via a standard computability-theoretic construction  for hyperimmune-free sets. We include this proof, which requires only straightforward changes from the original proof by Miller and Martin \cite{MillerMartin} of the existence of hyperimmune-free degrees.
	
	\begin{notation}
		The symbol $\demp$ represents the Turing degree that contains the computable sets.
		
		For a Turing degree $d$, the symbol $d'$ represents the result of applying the Halting jump operator to the Turing degree $d$.
		
		Generally, $d^{(n)}$ is the result of applying the Halting jump operator $n$ times to the Turing degree $d$.
	\end{notation}
	
	\begin{theorem}\label{HIFs exist}
		For any Turing degrees $c\geq d'$, there exists a $d$-hyperimmune-free degree $a$ so $d<_T a<_T c'$, and $a\not\leq_T c$.
	\end{theorem}
	
	\begin{proof}
		We employ a forcing argument with $d$-computable function-trees.
		\begin{definition}
			A function-tree is a function $T:2^{<\nat}\rightarrow 2^{<\nat}$ so that $\sigma\preceq \rho$ if and only if $T(\sigma)\preceq T(\rho)$.
			
			We say that $\sigma$ is on the tree $T$ if there is some $\alpha$ so that $\sigma=T(\alpha)$.
			We say that $\hat{T}$ refines $T$ if every $\alpha$ which is on $\hat{T}$ is also on $T$.
		\end{definition}
		
		We build a sequence of $d$-computable function-trees $T_i$ for $i\in \nat$ so that $T_{i+1}$ always refines $T_i$. Though each tree is $d$-computable, the sequence will be $c'$-computable.
		
		Then the sequence $T_i(\lambda)$ where $\lambda$ is the empty string satisfies $T_{i}(\lambda)\preceq T_{i+1}(\lambda)$. Let $Y$ be the unique set so that $j\in Y$ if and only if $T_i(\lambda)(j)=1$ for some or equivalently all sufficiently large $i$. This $Y$ will be our hyperimmune-free set. Since the sequence of trees is $c'$-computable, it follows that $Y$ is $c'$-computable.
		
		Fix $D$ to be any set in the degree $d$.
		We begin with $T_0$ being the $d$-computable function-tree $T_0(\sigma)=\sigma(0)D(0)\sigma(1)D(1)\dots \sigma(\vert \sigma\vert -1)D(\vert \sigma\vert -1)$. Since the set $Y$ we construct is a path through $T_0$, we have already ensured that $Y\geq_T d$. At even stages we extend to ensure that $T_{2n}(\lambda)$ is incompatible with $\phi_n^c$. That is, one of $T_{2n-1}(0)$ or $T_{2n-1}(1)$ is incompatible with $\phi_n^c$. Using $c'$ we can determine which. Suppose $T_{2n-1}(0)$ is incompatible with $\phi_n^c$. Then we define $T_{2n}(\sigma)=T_{2n-1}(0^\smallfrown \sigma)$. These stages collectively ensure that $Y$ is not computable from $c$.
		
		At odd stages, we ensure that $Y$ is $d$-hyperimmune-free by forcing that either the $n$th computable function $\phi_n$ is total on $T_{2n+1}$ or that $\phi_n$ is partial on $T_{2n+1}$. In the former case, this will ensure that $\phi_n^Y$ is $d$-computably-dominated. In the latter, this will ensure that $\phi_n^Y$ is partial.
		
		To do so, we define a function-tree $\hat{T}$ as follows: We let $\hat{T}(\lambda)=T_{2n}(\lambda)$. Then for each $\sigma$, we search if there is some $\alpha$ so that $\hat{T}(\sigma)\preceq T_{2n}(\alpha)$ so that $\phi_n^\alpha(|\sigma|)\downarrow$. If found, then we let $\hat{T}(\sigma^\smallfrown i)=T_{2n}(\alpha^\smallfrown i)$ for each $i\in \{0,1\}$.
		Using $d''$, which is $\leq_T c'$, we can check if $\hat{T}$ defines a total function-tree. If so, we let $T_{2n+1}=\hat{T}$. If not, then there is some $\alpha$ and $m$ so that no $\beta\succeq \alpha$ has $\phi_n^\beta(m)\downarrow$. In this case we let $T_{2n+1}(\sigma)=T_{2n}(\alpha^\smallfrown \sigma)$ and we have ensured that $\phi_n^Y(m)\uparrow$.
		
		Finally, we need to check that $Y$ is hyperimmune-free. Fix $e\in \nat$. We must see that if $\phi_e^Y$ is a total function, then it is dominated by a $d$-computable function. If at step $2e+1$ we did not make $T_{2e+1}=\hat{T}$, then we  ensured that $\phi_e^Y(m)$ does not converge, so $\phi_e^Y$ is not a total function. On the other hand, if we made $T_{2e+1}=\hat{T}$, then we ensured that $\phi_e^{\hat{T}(\sigma)}(m)$ converges for each $\sigma$ of length $m+1$. Thus $\phi^Y_e(m)\leq \max\{\phi_e^{\hat{T}(\sigma)}(m) : |\sigma|=m+1\}$. This gives a $d$-computable function (since $\hat{T}$ is $d$-computable) bounding $\phi_e^Y$.
	\end{proof}

	\begin{theorem}[Friedberg \cite{friedberg1957criterion}]\label{FriedbergJumpInversion}
		Let $d$ be any Turing degree above  $\demp'$. Then there exists another Turing degree $c$ so that $d=c'$.
	\end{theorem}

	\PIINOTUSWP*
	\begin{proof}
		Let $Y\geq_T \demp'$ be a $\demp^{(6)}$-computable set which is not $\demp^{(5)}$-computable and is $\demp'$-hyperimmune free; we apply Theorem \ref{HIFs exist} with $d=\demp'$ and $c=\demp^{(5)}$ to get such a set $Y$. Let $e$ be the degree of a jump-inverse of $Y$ as in Theorem \ref{FriedbergJumpInversion}. Since $e$ is computable from $\demp^{(6)}$, every set in $e$ is $\Pi^1_1$.
		
		Let $X$ be any set in $e$. We claim that $X$ is not \uswp. Suppose towards a contradiction that there were a computable sequence of trees $\T_i$ witnessing that $X$ was \uswp. Then $i\in X$ if and only if $Y$ computes a path through $\T_i$. But note that since $Y$ is $\demp'$-hyperimmune free, $Y$ computing a path through $\T_i$ implies that there is a $\demp'$-computable total function $g$ so that $\T_i \cap \{\sigma \in \nat^{<\nat} : \forall n<\vert \sigma \vert\, (\sigma(n)<g(n))\}$ has a path. Thus, $i\in X$ if and only if $\exists g\leq_T \demp' \text{ total so } \T_i^g:= \T_i \cap \{\sigma \in \nat^{<\nat} : \forall n<\vert \sigma \vert \sigma(n)<g(n)\} \text{ has a path}$. 
		We can now apply K\"onig's Lemma, which states that any finitely branching tree which has arbitrarily long strings must also have a path. Since $\T_i^g$ is finitely branching, it having a path is equivalent to it containing arbitrarily long strings. The tree $T_i^g$ is uniformly computed from $g$, thus from $\demp'$. Thus $\demp^{(2)}$ can check the condition that the tree has arbitrarily long strings, and $\demp^{(3)}$ can check that $g$ is total. Finally, $\demp^{(4)}$ can check the existence of a $g$ satisfying the full condition. Thus $X$ is computable in $\demp^{(4)}$ showing that $Y$ is computable in $\demp^{(5)}$, which is a contradiction. 
	\end{proof}

\end{document}